\documentclass[conference]{IEEEtran}
\IEEEoverridecommandlockouts
\newcommand{\numauthorrefmark}[1]{\textsuperscript{#1}}
\usepackage{cite}

\usepackage{mathrsfs}
\usepackage{amsmath,amssymb,amsfonts}
\usepackage{graphicx}
\usepackage[caption=false,font=footnotesize]{subfig}
\usepackage{textcomp}
\usepackage{url}
\usepackage{mathtools}
\usepackage{bbm}
\usepackage{accents}
\usepackage[dvipsnames]{xcolor}
\usepackage{amsthm}
\usepackage{algorithm}
\usepackage{algpseudocode}
\usepackage{hyperref}
\usepackage{cleveref}
\usepackage{enumitem} \definecolor{MajorNotesColor}{HTML}{2a8c2f}

\newcommand{\argmin}{\operatornamewithlimits{argmin}}

\DeclareDocumentCommand\vectorbold{ s m }{\IfBooleanTF{#1}{\boldsymbol{#2}}{\mathbf{#2}}}  

\newcommand{\nom}{\mathrm{nom}}
\newcommand{\pinom}{\pi_{\nom}}
\newcommand{\pifallback}{\pi_{\mathrm{fb}}}
\newcommand{\pifb}{\pifallback}
\newcommand{\piadv}{\pi_{\mathrm{adv}}}
\newcommand{\Vadv}{V_{\mathrm{adv}}}

\newcommand{\Jfb}{J^{\pifallback}}

\newcommand{\Dadmissible}{\mathcal{D}_{\mathrm{feas}}}

\newcommand{\Vhj}{V_{\mathrm{HJ}}}

\newcommand{\XSet}{\mathcal{X}}

\newcommand{\T}{^{\top}}

\DeclarePairedDelimiterX{\abs}[1]{\lvert}{\rvert}{#1}
\DeclarePairedDelimiterX{\norm}[1]{\lVert}{\rVert}{#1}

\definecolor{AlgColorLEAP}{HTML}{0D47A1}     \definecolor{AlgColorBaseline}{HTML}{4A148C} 

\newcommand{\AlgLEAP}{\textcolor{AlgColorLEAP}{\textsf{LEAP}}}

\newcommand{\AlgHJ}{\textcolor{AlgColorBaseline}{\textsf{HJ}}}
\newcommand{\AlgHJI}{\textcolor{AlgColorBaseline}{\textsf{HJI}}}
\newcommand{\AlgHCBF}{\textcolor{AlgColorBaseline}{\textsf{HCBF}}}
\newcommand{\AlgRHCBF}{\textcolor{AlgColorBaseline}{\textsf{R-HCBF}}}

\newcommand{\crowdnav}{\textsf{Crowd Navigation}}
\newcommand{\posswap}{\textsf{Position Swapping}}

\newcommand{\dubinsra}{\textsf{Dubins Reach-Avoid}}
\newcommand{\quadrupeddrone}{\textsf{Quadruped-Drone}}
 \newcommand*\diff{\mathop{}\!\mathrm{d}}

\setlist[enumerate]{leftmargin=*}
\setlist[itemize]{leftmargin=*}

\theoremstyle{definition}
\newtheorem{definition}{Definition}

\theoremstyle{plain}
\newtheorem{theorem}{Theorem}

\newtheorem{corollary}{Corollary}
\crefname{figure}{fig.}{figs.}
\Crefname{figure}{Fig.}{Figs.}
\crefname{corollary}{cor.}{cors.}
\Crefname{corollary}{Cor.}{Cors.}
\theoremstyle{remark}

\begin{document}
\bstctlcite{noUrl}
\title{LEAP-CBF: A Safety Filter for Uncertain Systems with Least-Effort Adversarial Potentials}
\author{Oswin So\numauthorrefmark{1*}, Eric Yu\numauthorrefmark{1*}, Chuchu Fan\numauthorrefmark{1} \thanks{\numauthorrefmark{1} Massachusetts Institute of Technology.}
\thanks{\numauthorrefmark{*} Equal Contribution.}
}
\maketitle
\thispagestyle{plain}
\pagestyle{plain}
\begin{abstract}
Control barrier functions (CBF) are a popular safety filter to ensure safety for nonlinear dynamical systems.
However, when the system is subject to uncertainties and disturbances, this requires the use of robust variants of CBFs,
which can be difficult to construct and can be overly conservative, especially for high-dimensional systems under input constraints.
In this work, we propose a new approach to solve these challenges by introducing Least-Effort Adversarial Potentials (LEAP), a certificate that quantifies the robustness of a given state against disturbances in terms of the effort required by the disturbance to cause failure.
We show that LEAP is a CBF for the undisturbed system, but can also be used to construct a safety filter that is robust to disturbances whose cumulative effort is bounded.
We propose a method for constructing LEAPs with on-policy deep reinforcement learning.
Next, we demonstrate LEAPs in simulation on a variety of multi-agent systems with disturbances and uncertainties.
Finally, hardware experiments on a quadruped and quadrotors validate that LEAPs are well suited to tackle the disturbances and uncertainties from real-world robotic systems.

 \end{abstract}
\section{Introduction}
\label{sec:introduction}

Robotic systems face unmodeled dynamics, sensor noise, and environmental disturbances~\cite{thrun2002probabilistic}.
A robot navigating among erratically moving agents must avoid collisions while making task progress.
Safety filters modify task commands when safety is threatened~\cite{ames2016control,hsu2023safety}, using a certificate of the robot's ability to avoid failure.
But should this certificate measure clearance or tolerance to disturbances?

Control barrier functions (CBFs) \cite{ames2019cbfsurvey} and Hamilton--Jacobi (HJ) reachability \cite{choi2021cbvf} provide principled tools for constructing safety filters.
However, a safety value need not directly quantify tolerance to disturbances.
For example, an HJ value defined through minimum signed distance along a trajectory accounts for dynamics, but its magnitude measures clearance rather than the disturbance effort needed to cause failure.
As illustrated in \Cref{fig:teaser}, a robot moving away from a nearby obstacle may be difficult to drive into it despite its small clearance.
Robust HJ formulations commonly certify safety against disturbances in a prescribed pointwise set~\cite{choi2021cbvf,hsunguyen2023isaacs}.
A pointwise disturbance bound permits the adversary to act persistently at that bound.
When disturbances are brief but large in magnitude, the resulting safety filter can be overly conservative, restricting task progress.
Energy-bounded reachability and IQC-based CBFs already account for cumulative or dynamic uncertainty~\cite{yin2020reachability,seiler2021control}.
This motivates an explicit effort-to-failure certificate that connects cumulative disturbance budgets to runtime filtering.

\begin{figure}[!t]
    \centering
\includegraphics[width=\linewidth]{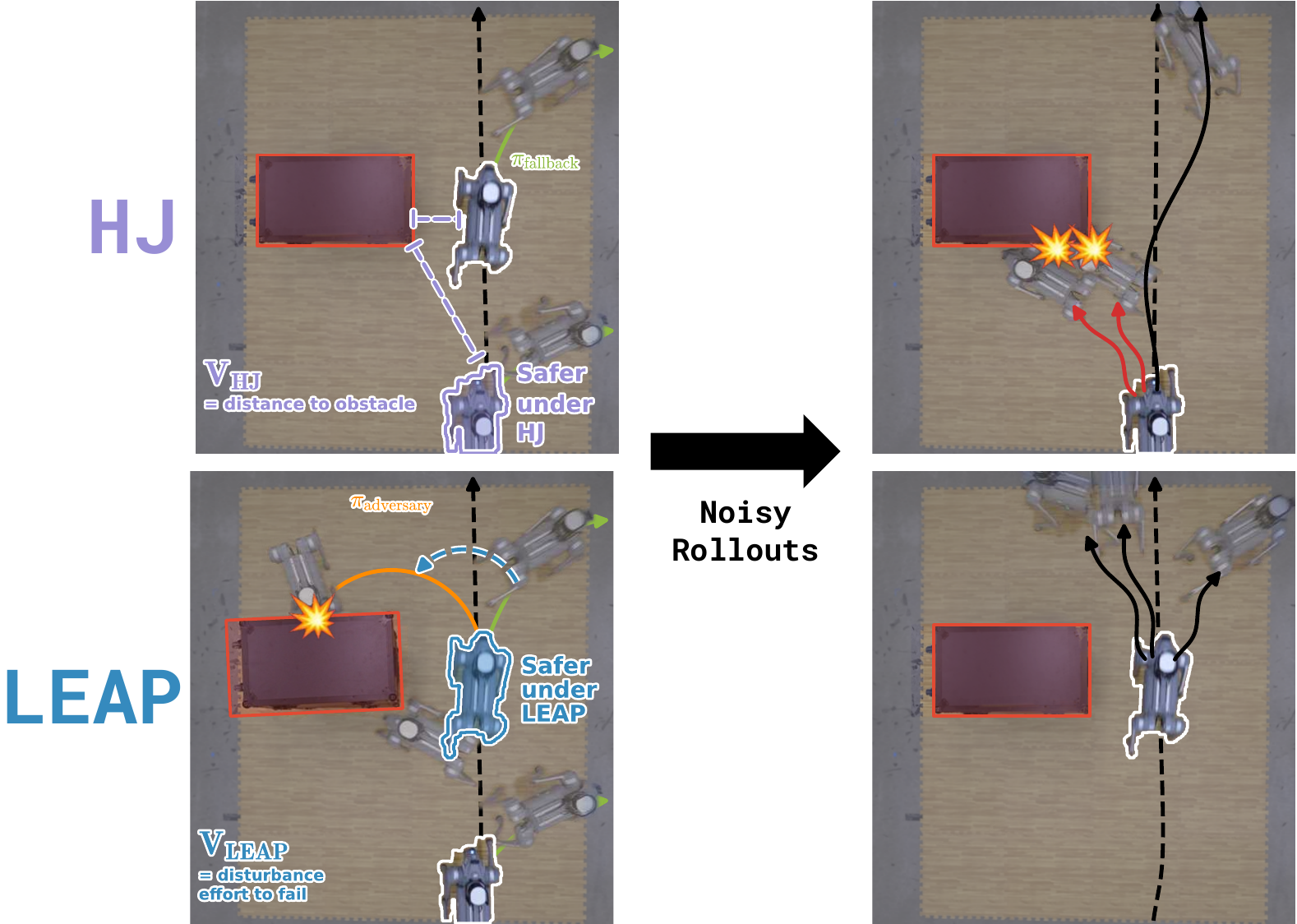}
\caption{\textbf{LEAP measures safety by the effort required to cause failure.}
    A quadruped, modeled as a unicycle, moves past a box obstacle (red).
    Although its geometric clearance is small, causing a collision can require substantial disturbance effort to redirect its motion.
    LEAP quantifies this effort under a fallback policy, providing a robustness margin that reflects the dynamics and disturbance model.
    }
    \label{fig:teaser}
\end{figure}

We introduce the \textbf{Least-Effort Adversarial Potential (LEAP)}: the minimum cumulative disturbance effort required to cause failure under a fixed fallback policy.
Computing LEAP is a minimum-cost reachability problem with the disturbance as the sole optimization variable.
LEAP-CBF uses this potential to filter task commands under a cumulative disturbance budget.
Under appropriate regularity assumptions, positive-threshold shifts of the exact LEAP value are CBFs for the undisturbed dynamics.
We further derive safety filters that guarantee safety under a cumulative disturbance budget.
To approximate LEAP in practice, we develop an on-policy reinforcement learning procedure that alternates adversary training and fallback improvement, followed by a final value-estimation stage with the fallback fixed.

\noindent\textbf{Robust safety under integral uncertainty.}
Robust and input-to-state safe CBFs account for disturbances~\cite{xu2015robustness,jankovic2018robust,kolathaya2019input}, often characterized by pointwise bounds.
Conservatism can be reduced through tunable disturbance accommodation~\cite{alan2022safe} or learning the effect of model uncertainty on the barrier derivative~\cite{taylor2020learning}.
Measurement-robust CBFs address state-estimation errors~\cite{dean2021guaranteeing}, a distinct uncertainty source from the disturbances entering LEAP's dynamics.
Cumulative uncertainty is also well established through energy-bounded reachability and integral quadratic constraints (IQCs)~\cite{yin2020reachability,yin2021backward}, IQC-based CBFs~\cite{seiler2021control}, and integral input-to-state safe barrier functions~\cite{lyu2025integral}.
LEAP builds on this perspective by quantifying the minimum cumulative disturbance effort required to cause failure under a fixed fallback, yielding an explicit robustness margin for disturbance-budget filtering.

\noindent\textbf{Minimum-cost reachability and failure search.}
Reachability is closely connected to optimal control~\cite{lygeros2004reachability}: running-cost HJB formulations characterize cost-limited reachable sets~\cite{liao2021costlimited}, and RL methods address minimum-cost reach-avoid problems~\cite{so2024solving}.
Adaptive stress testing optimizes a simulator's stochastic inputs to find likely failure trajectories~\cite{lee2020adaptive}.
LEAP likewise targets failure under specified system behavior, but minimizes cumulative disturbance effort and uses the resulting value as a certificate for online safety filtering.
Its contribution is this connection to nominal CBFs and disturbance-budget safety, rather than introducing running-cost reachability.

\noindent\textbf{Policy-based and predictive safety filters.}
Backup and policy CBFs construct certificates around specified policies~\cite{chen2021backup,so2024train}, with robust extensions addressing disturbances and input constraints~\cite{knoedler2025safety}.
Connections between learned value functions and CBFs have also been established under suitable task structures~\cite{tan2023value}.
Predictive safety filters use online model-based optimization to enforce constraints~\cite{wabersich2021predictive}; auxiliary soft-constrained predictive problems can themselves define barriers and provide recovery toward the feasible set~\cite{wabersich2023predictive}.
Adversarial predictive rollouts offer another runtime filtering mechanism~\cite{nguyen2024gameplayfiltersrobustzeroshot}.
Stochastic CBFs have also been constructed from a learned adversarial time-to-collision metric, predicting collision between a nonreactive evader and a dynamically constrained pursuer~\cite{sorensen2026time}.
This certificate measures the time available before collision, whereas LEAP quantifies a cumulative disturbance budget under a fixed fallback policy.
For control-affine dynamics with additive disturbances and quadratic effort, its precomputed value yields a value-gradient QP requiring no online adversarial rollout.

\noindent\textbf{Learned adversarial certificates.}
Joint controller--adversary training is established in robust RL~\cite{pinto2017robust}, with safety-oriented formulations learning robust game strategies~\cite{hsunguyen2023isaacs} and state--action barrier certificates~\cite{oh2026synthesis}.
Neural HJ solvers~\cite{bansal2021deepreach} and MPC-guided adversarial supervision~\cite{teoh2026madr} provide alternative approaches to approximating safety values.
LEAP uses cumulative-return RL and fixes the fallback for its final effort-value estimation.
Reachability-based refinement~\cite{tonkens2022refining} and local patching of approximately safe values~\cite{tonkens2024patching} address certificate approximation errors.
Our guarantees concern the exact LEAP certificate; critic fitting and polishing are evaluated empirically and do not provide such verification.

\noindent\textbf{Contributions.} We summarize our contributions as follows:
\begin{enumerate}
    \item We formulate LEAP as a minimum-disturbance-effort reachability value and establish its connection to CBFs and safety under a cumulative disturbance budget.
    \item We derive LEAP-based safety filters and develop a deep reinforcement learning procedure to approximate the potential and learn its fallback policy.
    \item We evaluate LEAP-CBF on noisy Dubins navigation, adversarial crowd navigation, and multi-agent position swapping with tracking latency, and on a Unitree Go2 avoiding a Crazyflie drone. The experiments show improved safety--task performance tradeoffs relative to the evaluated HJ and CBF baselines.
\end{enumerate}
 \section{Problem Setup}
\label{sec:preliminaries}
We consider continuous-time dynamics with disturbances as
\begin{equation} \label{eq:dyn}
    \dot{x}_t = f(x_t, u_t, d_t),
\end{equation}
with state $x_t\in\mathcal X\subset\mathbb R^{n_x}$, control $u_t\in\mathcal U\subset\mathbb R^{n_u}$, and measurable disturbance $d_t\in\mathcal D\subseteq\mathbb R^{n_d}$, possibly unbounded.
The goal is to avoid a failure set $\mathcal F\subset\mathcal X$ despite unknown disturbances.
Assuming $0\in\mathcal D$, define the nominal system by
\begin{equation} \label{eq:nominal_dyn}
    \dot{\bar{x}}_t=\bar f(\bar{x}_t,u_t),\qquad \bar f(x,u)\coloneqq f(x,u,0).
\end{equation}
\begin{definition}[Control Barrier Function]
    A continuously differentiable function $B: \XSet \to \mathbb{R}$ is a (zeroing) Control Barrier Function (CBF) for the nominal system \eqref{eq:nominal_dyn} if there exists an extended class-$\mathcal{K}$ function $\alpha$ \cite{ames2016control} such that
    \begin{subequations}
    \begin{align}
        B(x) &< 0, \quad \forall x \in \mathcal{F}, \label{eq:cbf:sign} \\
        B(x) &\geq 0 \implies \sup_{u \in \mathcal{U}} \nabla B(x)\T \bar{f}(x, u) \geq -\alpha(B(x)). \label{eq:cbf:descent}
    \end{align}
    \end{subequations}
\end{definition}
Any sufficiently regular policy satisfying \eqref{eq:cbf:descent} renders $\mathcal S=\{x:B(x)\geq0\}$ forward invariant~\cite{ames2016control}.
For control-affine dynamics $\bar f(x,u)=f_0(x)+g(x)u$, where $f_0:\mathcal X\to\mathbb R^{n_x}$ and $g:\mathcal X\to\mathbb R^{n_x\times n_u}$, this condition is linear in $u$, giving the CBF-QP safety filter for nominal command $\pinom(x)$:
\begin{subequations} \label{eq:cbf_qp}
\begin{align}
    \argmin_{u \in \mathcal{U}} \quad & \frac{1}{2}\|u - \pi_{\mathrm{nom}}(x)\|_2^2 \label{eq:cbf_qp:objective}\\ \text{s.t.} \quad & \nabla B(x)^\top \bigl(f_0(x) + g(x)u\bigr) \geq -\alpha(B(x)). \label{eq:cbf_qp:constraint}
\end{align}
\end{subequations}

\begin{figure}
    \centering
    \includegraphics[width=\linewidth]{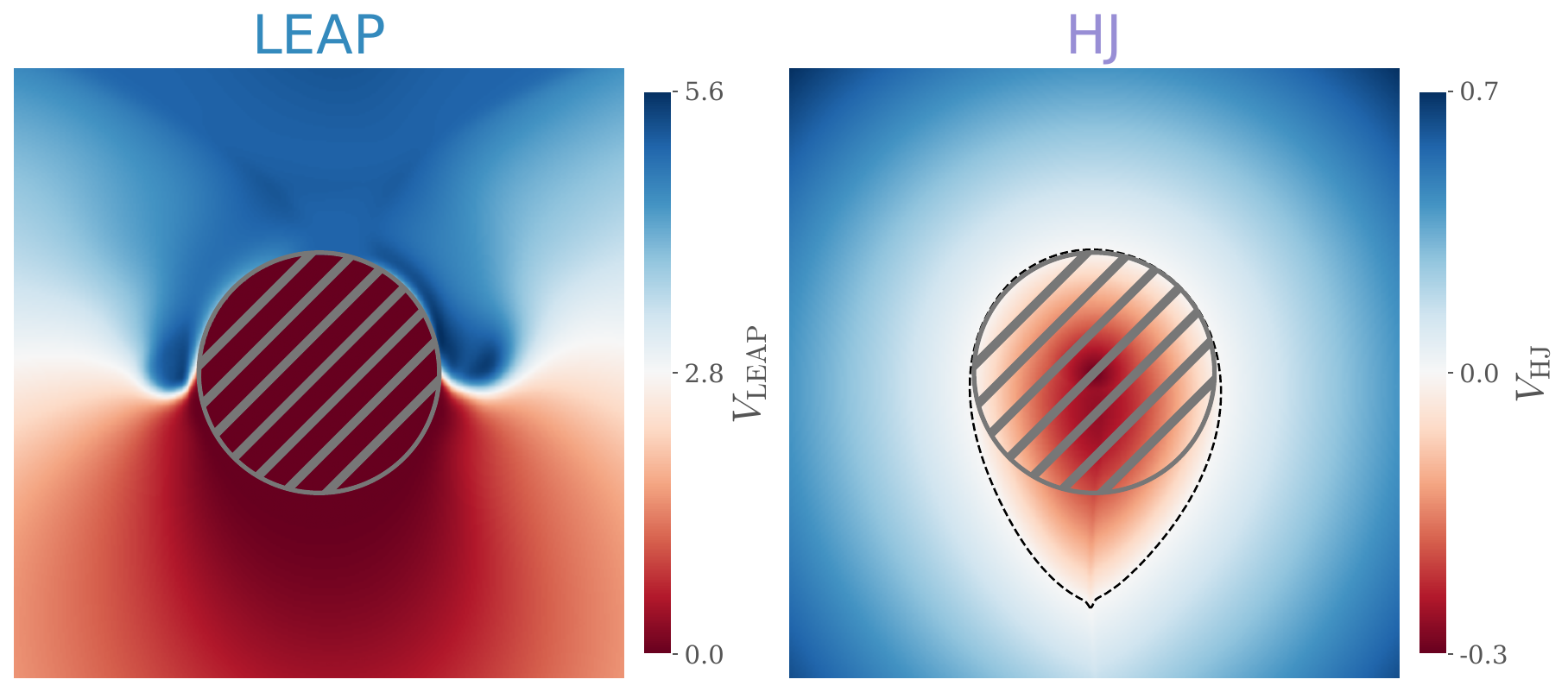}
    \caption{
        \textbf{LEAP vs HJ value functions for a Dubins car.}
        We visualize the LEAP value $\Jfb$ and the Hamilton-Jacobi (HJ) value $\Vhj$ for a Dubins car pointed up with an circular obstacle at the origin, where the disturbance is additive in the steering.
        $\Vhj$ uses the signed distance to the obstacle as the surface function $h$.
        LEAP puts states above the obstacle as the most robust, since colliding into the obstacle requires a large disturbance effort to turn the car around.
        The HJ value has a strong distance bias and does not distinguish between states above and below the obstacle once far enough away.
    }
    \label{fig:leap_visualization}
\end{figure}

\begin{figure*}[t]
    \centering
    \includegraphics[width=\linewidth]{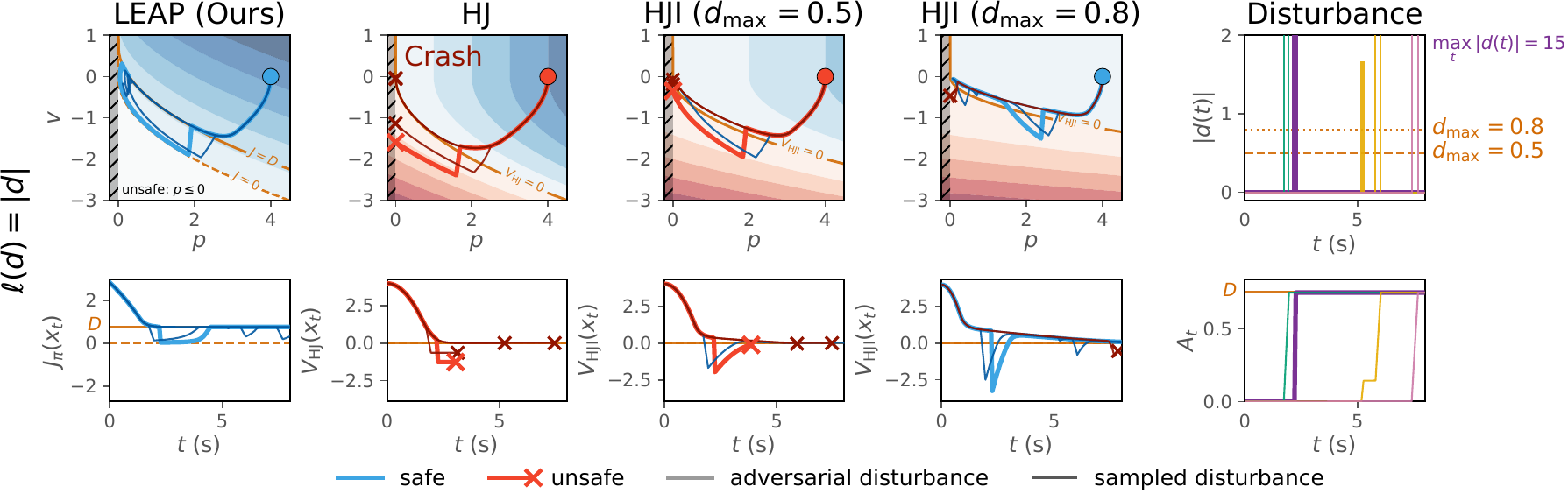}
\caption{\textbf{Cumulative-effort budgets accommodate short, large disturbance peaks.}
    A double integrator $\dot p=v$, $\dot v=u+d$, with $p\leq0$ unsafe, receives the same disturbances under LEAP ($\pi_{\mathrm{sw+qp}}$, \Cref{thm:switch_qp}), HJ, and HJI filters.
    Top: phase trajectories over the respective certificates. Bottom: certificate values over time.
    Right: disturbance magnitudes and accumulated absolute effort $A_t$.
    Peaks exceed the assumed HJI bounds, while $A_t<D$.
    LEAP remains safe in the displayed rollouts; HJ and both HJI configurations admit collisions.
    Blue/red indicate safety/collision, and crosses mark collisions.
    Crossing a certificate's zero contour does not itself imply collision.}
    \label{fig:dbint_analytic_compare}
\end{figure*}

\section{Safety for Undisturbed Systems}
\label{sec:leap}

For a fixed fallback policy $\pifallback$ and nonnegative effort rate $\ell$ with $\ell(0)=0$, define the LEAP value as
\begin{subequations} \label{eq:leap_value}
\begin{align}
\Jfb (x_0)
\coloneqq
\inf_{d(\cdot)} \quad
& \int_0^T \ell(d_t)\, dt \label{eq:leap_value:obj} \\
\text{s.t.} \quad
& \dot{x}_t = f\bigl(x_t, \pifallback(x_t), d_t\bigr), \label{eq:leap_value:dynamics} \\
& T \coloneqq \inf\{t \geq 0 : x_t \in \mathcal{F}\} < \infty.
\end{align}
\end{subequations}
where the infimum over the empty set is $+\infty$.
Thus $\Jfb\geq0$, with $\Jfb=0$ on $\mathcal F$ and $\Jfb=+\infty$ when no disturbance can cause failure.
\eqref{eq:leap_value} is a running-cost reachability problem~\cite{liao2021costlimited,so2024solving} in which the disturbance is the optimization variable and the fallback is fixed. \Cref{fig:leap_visualization} compares LEAP with a nominal signed-distance HJ value for a Dubins car.

\noindent \textbf{LEAP as a safety certificate for the nominal systems.}
An important property of $\Jfb$ is that it is a CBF under the nominal dynamics \eqref{eq:nominal_dyn}, which we show in the following theorem.
\begin{theorem}
    For any $D>0$, let $\Jfb_D(x) \coloneqq \Jfb(x) - D$ and $\mathcal{S}_D \coloneqq \{x : \Jfb_D(x) \geq 0\}$.
    Then, $\Jfb_D$ is a CBF for the nominal dynamics \eqref{eq:nominal_dyn}.
\end{theorem}
\begin{proof}
    Since $\Jfb=0$ on $\mathcal F$, $\Jfb_D(x) = -D < 0$ for all $x \in \mathcal{F}$, which satisfies the first CBF condition \eqref{eq:cbf:sign}.
    For the second condition \eqref{eq:cbf:descent}, since $\Jfb$ is the value function of an optimal control problem, it satisfies the stationary Hamilton-Jacobi-Bellman (HJB) equation \cite{bardi1997optimal} on the complement of the failure set $\mathcal{F}^\complement = \XSet \setminus \mathcal{F}$.
    \begin{align}
        0
        &= \inf_d \{ \ell(d) + \nabla \Jfb(x)\T f(x, \pifallback(x), d) \}, \\
        &\leq  \ell(0) + \nabla \Jfb(x)\T f(x, \pifallback(x), 0), \\
        &\leq \sup_u \nabla \Jfb(x)\T \bar{f}(x, u).
    \end{align}
    For any $x \in \mathcal{S}_D$, $\alpha(\Jfb_D(x)) \geq 0$. Thus,
    \begin{equation}
        \sup_u \nabla \Jfb_D(x)\T \bar{f}(x, u) \geq 0 \geq -\alpha(\Jfb_D(x)),
    \end{equation}
    and the second CBF condition \eqref{eq:cbf:descent} holds.
\end{proof}
Consequently, using $\Jfb$ to construct a CBF-QP safety filter \eqref{eq:cbf_qp} guarantees that the set $\mathcal{S}_D \subseteq \mathcal{F}^\complement$ is forward invariant.

\noindent\textbf{Interpreting the robustness of the HJ value.}
For the nominal HJ reachability formulation considered here, the value
$\Vhj(x)=\sup_\pi\inf_{t\geq0}h(\bar{x}_t)$
measures the best achievable minimum of a chosen surface function
$h$ along trajectories~\cite{bansal2017hamilton}.
Although $\{x:h(x)\leq0\}=\mathcal F$ specifies failure, the
value's magnitude depends on the choice of $h$.
With signed distance, it measures clearance along trajectories,
which need not reflect tolerance to disturbances.
\Cref{fig:teaser} illustrates this distinction:
HJ and LEAP rank states differently, and the state favored by
HJ experiences more collisions in the displayed noisy rollouts.
In contrast, LEAP's magnitude measures the minimum cumulative
disturbance effort required for failure under the fallback,
giving its thresholds a direct disturbance-budget interpretation.

\section{Safety for Disturbed Systems}
We now return to the original problem of ensuring safety under disturbances.
One way to specify this problem is through a pointwise disturbance set, as in standard robust HJ reachability~\cite{bansal2017hamilton}.
Following the broader study of integral uncertainty~\cite{yin2020reachability,seiler2021control}, we instead express robustness through the cumulative disturbance effort in the LEAP value~\eqref{eq:leap_value}.

For a disturbance signal $(d_t)_{t \geq 0}$, we define the \textit{disturbance budget} $A_t$ as the total accumulated disturbance effort $\ell$, i.e.,
\begin{equation} \label{eq:disturbance_budget}
    A_t = \int_0^t \ell(d_s) \diff s.
\end{equation}
With quadratic effort, \eqref{eq:disturbance_budget} measures disturbance energy, while absolute effort measures cumulative disturbance magnitude.

Let $D > 0$ be a fixed disturbance budget.
We now consider the problem of ensuring safety under disturbances whose total effort is bounded by $D$.
In other words, we define the set of admissible disturbance signals $\Dadmissible$ as
\begin{equation}
    \Dadmissible = \{(d_t)_{t \geq 0} : A_t < D, \forall t \geq 0\}.
\end{equation}
We define the notion of disturbance-budget safety as follows.
\begin{definition}[Disturbance-Budget Safety]
    For a given disturbance budget $D > 0$, a system started from $x_0$ is safe under budget $D$ if it does not enter the failure set $\mathcal{F}$ for any disturbance signal $(d_t)_{t \geq 0} \in \Dadmissible$.
\end{definition}

Unlike the pointwise amplitude bound $|d_t|\leq d_{\max}$ used by HJI, an effort budget permits short, large peaks but limits their cumulative effect.
\Cref{fig:dbint_analytic_compare} illustrates this distinction: the disturbances exceed the assumed HJI bounds while satisfying $A_t<D$ for $\ell(d)=|d|$.
LEAP remains safe in the displayed rollouts, whereas HJ and HJI admit collisions outside their disturbance assumptions.
While HJI can be applied despite the mismatch with its assumed pointwise amplitude bounds, this mismatch leads to conservative behavior and higher intervention effort (\Cref{fig:dbint_analytic_hji_vs_leap}). 

\begin{figure}[t]
    \centering
    \includegraphics[width=\linewidth]{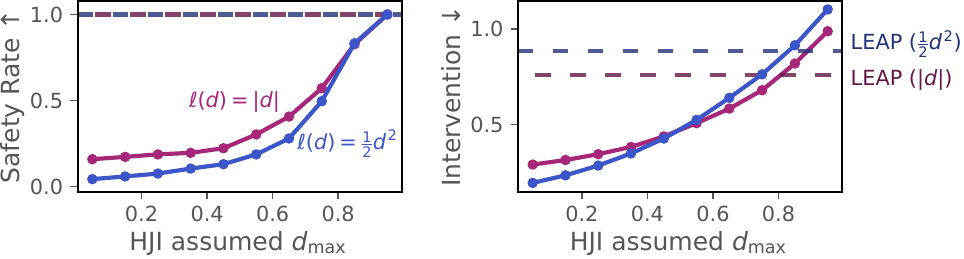}
    \caption{\textbf{Safety and intervention under effort-bounded disturbances.}
Larger HJI bounds improve observed safety but increase intervention effort.
    LEAP attains the same observed safety rate as the safest evaluated HJI configurations with less intervention.}
    \label{fig:dbint_analytic_hji_vs_leap}
\end{figure}

\begin{figure}[t]
    \centering
    \includegraphics[width=\linewidth]{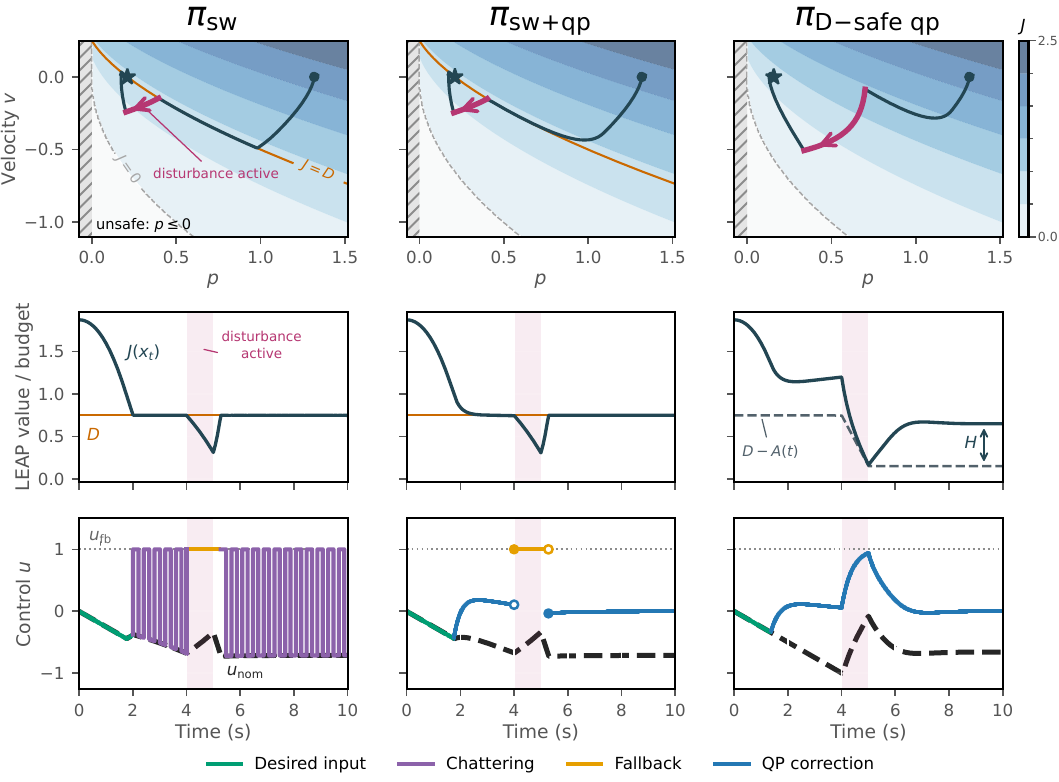}
    \caption{\textbf{Three ways to enforce disturbance-budget safety.}
We compare $\pi_{\mathrm{sw}}$ \eqref{eq:budget_switch}, $\pi_{\mathrm{sw+qp}}$ \Cref{thm:switch_qp}, and the budget-aware QP \eqref{eq:budget_qp} on the same scenario as \Cref{fig:dbint_analytic_compare}.
    $\pi_{\mathrm{sw}}$ chatters near $J=D$.
  Adding a nominal CBF-QP $\pi_{\mathrm{sw+qp}}$ smooths intervention but retains fallback switching when $J\leq D$.
  The budget-aware QP $\pi_{\mathrm{budget}}$ preserves $J\geq D-A_t$, but requires tracking the accumulated disturbance effort $A_t$.
    }
    \label{fig:dbint_analytic_leap_mechanisms}
\end{figure}

Given the similarities between the LEAP value \eqref{eq:leap_value} and the disturbance-budget safety definition,
it is natural to ask whether the LEAP value can also be used to construct a safety filter that guarantees disturbance-budget safety.
For convenience, for $D > 0$, let $\Jfb_D(x) \coloneqq \Jfb(x) - D$ and $\mathcal{S}_D \coloneqq \{x : \Jfb_D(x) \geq 0\}$.
We now answer in the affirmative with the following theorems.
\begin{theorem}[disturbance-budget safety certificate]\label{thm:budget}
    $\Jfb_D$ certifies the safety of the set $\mathcal{S}_D$ under budget $D$ in the sense that, under the fallback policy $\pifallback$, all states $x_0 \in \mathcal{S}_D$ are safe under budget $D$.
\end{theorem}
\begin{proof}
    If a disturbance in $\Dadmissible$ causes finite-time failure at $T$ under $\pifallback$ from $x_0 \in \mathcal{S}_D$, it is feasible in \eqref{eq:leap_value}, giving
    \begin{equation}
        D \leq \Jfb(x_0) \leq \int_0^T \ell(d_t)\,\diff t = A_T < D,
    \end{equation}
    which is a contradiction. This also holds for $\Jfb(x_0) = \infty$.
\end{proof}
This result enables the following switching safety filter that guarantees disturbance-budget safety. For a nominal policy $\pinom : \XSet \to \mathcal{U}$, define the switching policy
\begin{equation}\label{eq:budget_switch}
    \pi_{\mathrm{sw}}(x) \coloneqq
    \begin{cases}
        \pinom(x), & \Jfb_D(x) > 0, \\
        \pifallback(x), & \Jfb_D(x) \leq 0.
    \end{cases}
\end{equation}
We now prove that this switching policy guarantees disturbance-budget safety.
\begin{theorem}[Disturbance-budget safety of a switching filter]\label{thm:budget_switch}
    Suppose $\Jfb$ is finite and continuous on $\XSet$.
    Then, every $x_0 \in \mathcal{S}_D$ is safe under budget $D$ under $\pi_{\mathrm{sw}}$.
\end{theorem}
\begin{proof}
    Suppose a disturbance in $\Dadmissible$ drives the switched system from $x_0 \in \mathcal{S}_D$ to $x_T \in \mathcal{F}$ at some finite time $T$.
    Since $\Jfb(x_T)=0<D\leq\Jfb(x_0)$, continuity gives a last threshold crossing $s=\max\{t\in[0,T]:\Jfb(x_t)=D\}$.
    On $(s,T]$, $\Jfb(x_t)<D$, so the filter follows $\pifallback$.
    The shifted disturbance $\widetilde d_t=d_{s+t}$ belongs to $\Dadmissible$, since its accumulated effort satisfies $\widetilde A_t=A_{s+t}-A_s\leq A_{s+t}<D$ for all $t\geq0$.
    Under this disturbance, the fallback trajectory from $x_s\in\mathcal{S}_D$ reaches $x_T\in\mathcal{F}$, contradicting \Cref{thm:budget}.
\end{proof}
\begin{corollary} \label{thm:switch_qp}
    Consider the CBF-QP safety filter \eqref{eq:cbf_qp} with $\Jfb_D$ as the CBF, but switching to the fallback policy $\pifallback$ whenever the QP is infeasible or $\Jfb_D(x) \leq 0$.
    Call this controller $\pi_{\mathrm{sw+qp}}$.
    Then, for any $x_0 \in \mathcal{S}_D$, the system is safe under budget $D$ under this safety filter
\end{corollary}
Unlike the CBF-QP, the set $\mathcal{S}_D$ is \textit{not} forward invariant under $\pi_{\mathrm{sw}}$.
Intuitively, the extra margin $D$ adds just enough margin such that a disturbance signal would need to exceed the budget $D$ to drive the system into failure.
However, we can construct a variant of the CBF-QP that does not require switching to the fallback policy $\pifallback$ and still guarantees disturbance-budget safety, which we show next.
\begin{theorem}[Disturbance-budget CBF-QP filter]\label{thm:budget_qp}
    Suppose $\Jfb$ is finite and continuous on $\XSet$, and continuously differentiable wherever $\Jfb>0$.
    Let $H(x,A)\coloneqq\Jfb(x)-D+A$, with $A_t$ given by \eqref{eq:disturbance_budget}, and choose $\kappa>0$.
    Consider the filter $\pi_{\mathrm{D\text{-}safe\ qp}}(x,A)$ defined by the minimizer of
    \begin{equation}\label{eq:budget_qp}
        \begin{aligned}
\min_{u\in\mathcal{U}}\quad
            &\tfrac12\|u-\pinom(x)\|^2\\
            \text{s.t.}\quad
            &\inf_{d\in\mathcal{D}}\{\nabla\Jfb(x)^\top f(x,u,d)+\ell(d)\}\\
            &\qquad\geq-\kappa H(x,A).
        \end{aligned}
    \end{equation}
    The constraint is feasible whenever $H\geq0$ and $A<D$.
    From any $x_0\in\mathcal{S}_D$, this filter preserves $H(x_t,A_t)\geq0$ while $A_t<D$ and guarantees disturbance-budget safety.
\end{theorem}
\begin{proof}
    The fallback HJB inequality gives $\inf_d\{\nabla\Jfb(x)^\top f(x,\pifallback(x),d)+\ell(d)\}\geq0$.
    Thus $u=\pifallback(x)$ is feasible when $H\geq0$ and $A<D$, since then $\Jfb(x)>0$ and $-\kappa H\leq0$.
    Along the filtered trajectory, $\dot A_t=\ell(d_t)$ and \eqref{eq:budget_qp} imply
    \begin{equation}
        \dot H_t=\nabla\Jfb(x_t)^\top f(x_t,u_t,d_t)+\ell(d_t)
        \geq-\kappa H_t.
    \end{equation}
    Since $H_0=\Jfb(x_0)-D\geq0$, integration gives $H_t\geq e^{-\kappa t}H_0\geq0$ while $\Jfb(x_t)>0$ and $A_t<D$.
    By continuity, a first time with $\Jfb(x_t)=0$ and $A_t<D$ would give $H_t=A_t-D<0$, a contradiction.
    Hence $\Jfb(x_t)\geq D-A_t>0$ whenever $A_t<D$, which excludes failure.
\end{proof}
For control-affine dynamics with additive disturbances $f(x,u,d)=f_0(x)+g(x)u+E(x)d$, $\ell(d)=\tfrac12d^\top Rd$ is quadratic with $R\succ0$, and $\mathcal{D}=\mathbb{R}^{n_d}$, the disturbance minimization yields a linear constraint in $u$, which yields a QP when $\mathcal{U}$ is a polytope.
We illustrate the three disturbance-budget safety filters in \Cref{fig:dbint_analytic_leap_mechanisms}.

Given a LEAP $\Jfb$, we have constructed several safety filters that guarantee disturbance-budget safety.
In the next section, we propose a method of computing $\Jfb$ for general nonlinear systems using deep reinforcement learning, which is scalable to high-dimensional systems and does not require a custom solver for the min-over-time or max-over-time structure common with HJ-based approaches.

 \section{Solving LEAP with Deep Reinforcement Learning}
\label{sec:solve_leap}

\begin{algorithm}[t]
\caption{LEAP Training}
\label{alg:leap_training}
\begin{algorithmic}[1]
    \Require cost function $\ell$
    \State Randomly initialize fallback $\pifb$
    \For{$k=1,\ldots,K$}
        \State Train $(\piadv,\Vadv)$ using~\eqref{eq:rl_adv_obj}
        \State Train $\pifb$ using~\eqref{eq:rl_fb_obj}
    \EndFor
    \State Train final $(\piadv,\Vadv)$ using~\eqref{eq:rl_adv_obj}
    \State Polish $\Vadv$ by running deterministic rollouts of $\piadv$ under $\pifb$
    \State \Return $\pifb$ and $\Jfb \approx -\Vadv$
\end{algorithmic}
\end{algorithm}

We now describe how to solve the LEAP problem in practice.
Since the LEAP value is a optimal control problem with a sum-of-costs objective \eqref{eq:leap_value}, after discretizing time, we can apply standard deep reinforcement learning (RL) techniques.
This lets us approximate $\Jfb$ with standard cumulative-return RL, as in minimum-cost reach-avoid learning~\cite{so2024solving}; trajectory-extremum HJ objectives instead use corresponding Bellman updates or PDE solvers~\cite{hsunguyen2023isaacs,bansal2021deepreach}.

As shown in \Cref{alg:leap_training}, we propose an alternating optimization procedure using on-policy actor-critic RL (e.g., PPO~\cite{schulman2017proximal}).
Let $F_{\Delta t}(x,u,d)$ denote the one-step discretization of \eqref{eq:dyn}.
For a fixed fallback policy $\pifb$, the adversary searches for the least costly disturbance sequence that drives the system into failure.
Let $T$ denote the first time step at which $x_t\in\mathcal F$.
The adversary objective is
\begin{subequations}\label{eq:rl_adv_obj}
\begin{align}
    \piadv^\ast\in\argmin_{\piadv}\quad
    &\mathbb E\!\left[\sum_{t=0}^{T-1}
                    \gamma^t\ell(d_t)\Delta t\right] \\
    \text{s.t.}\quad
    &x_{t+1}=F_{\Delta t}\bigl(x_t,\pifb(x_t),d_t\bigr), \\
    &d_t\sim\piadv(\cdot\mid x_t).
\end{align}
\end{subequations}
We train $\piadv$ as a stochastic on-policy actor so that it explores nonzero disturbances rather than collapsing prematurely to $d_t=0$.
Because every nonzero disturbance incurs cost, the objective incentivizes the adversary to favor attacks that reach $\mathcal{F}$ with the least total effort, rather than attacks that simply use the largest disturbance at every step.

After each adversary update, we freeze its critic $\Vadv$ and train the fallback on nominal, undisturbed rollouts:
\begin{subequations}\label{eq:rl_fb_obj}
\begin{align}
    \pifb^\ast\in\argmin_{\pifb}\quad
    &\mathbb E\!\left[\max_{t\geq0}\Vadv(x_t)\right] \\
    \text{s.t.}\quad
    &x_{t+1}=F_{\Delta t}(x_t,u_t,0), \\
    &u_t\sim\pifb(\cdot\mid x_t).
\end{align}
\end{subequations}
Since $\Vadv(x)\approx-\Jfb(x)$, minimizing the largest critic value encountered along the rollout is equivalent to maximizing the smallest disturbance effort required to cause failure.

After $K$ alternating iterations, we retrain the adversary against the final fallback and polish its critic via deterministic policy evaluation.

\textbf{Improving the learned value function.}
We refine the critic to handle two cases: states where the fallback fails without disturbance, whose LEAP value is zero, and states where adversarial failure is unreachable, whose value is infinite.
From deterministic rollouts, we train a fallback-safety classifier $s_\eta(x)$ to predict failure under the undisturbed fallback and an adversary-reach classifier $c_\phi(x)$ to predict failure under the learned adversary--fallback pair.
With both policies fixed, we fit a critic $V_{\mathrm{feas}}$ on states classified as adversarially reachable using discounted $n$-step returns of the negative disturbance cost~\cite{sutton2018nstep}.
States in $\mathcal F\cup\{x:s_\eta(x)\geq\tau_s\}$ are treated as terminal with zero value, where $\tau_s$ is the fallback-failure threshold.
The deployed estimate combines this critic with a large finite cost cap $J_{\mathrm{cap}}>0$:
\begin{equation}
    \begin{aligned}
        \widehat J_{\pifb}(x)
        ={}&\mathbf{1}\!\left\{s_\eta(x)<\tau_s\right\}\bigl[
            c_\phi(x)\bigl(-V_{\mathrm{feas}}(x)\bigr)
            \\[-2pt]
            &{}+\bigl(1-c_\phi(x)\bigr)J_{\mathrm{cap}}\bigr].
    \end{aligned}
\end{equation}
Thus, predicted fallback-unsafe states receive zero effort, while states predicted unreachable by the adversary receive the cost cap.
We leave a formal convergence analysis of this coupled training procedure to future work.
Empirically, a single alternating iteration ($K=1$) is sufficient to obtain a good $\widehat{J}_{\pifb}$ across all of our experiments.

 \section{Simulation Experiments}
\label{sec:experiments}

We evaluate robustness to time-correlated steering noise, adversarial agent motion, and tracking latency in three simulation tasks (\Cref{fig:task_overview}).

\begin{figure}[H]
    \centering
    \subfloat[Steering noise.\label{fig:task_overview_dubins}]{\includegraphics[width=0.32\linewidth]{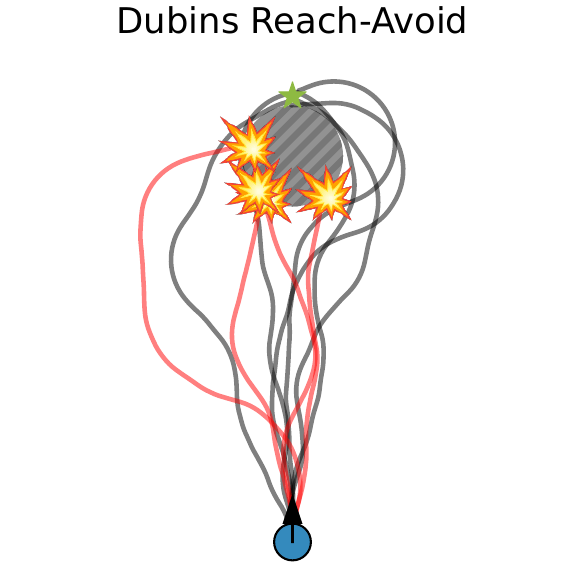}}
    \hfill
    \subfloat[Adversarial agent.\label{fig:task_overview_crowd}]{\includegraphics[width=0.32\linewidth]{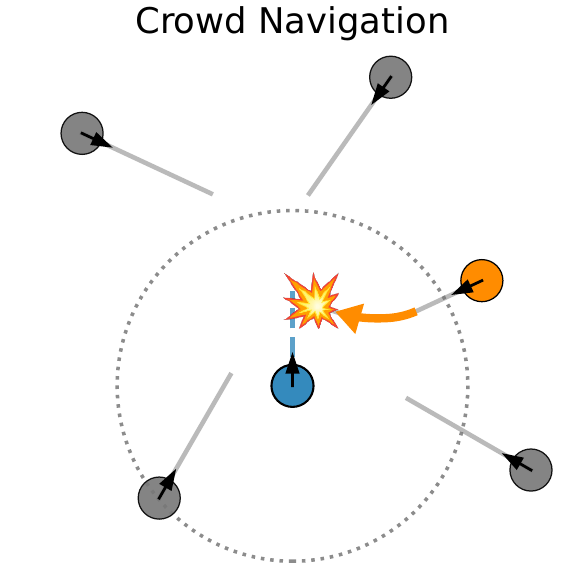}}
    \hfill
    \subfloat[Tracking latency.\label{fig:task_overview_swap}]{\includegraphics[width=0.32\linewidth]{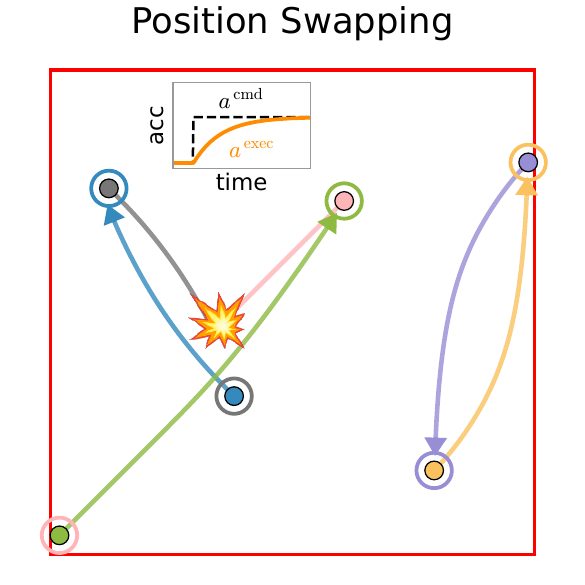}}

    \caption{\textbf{Simulation tasks and their test-time perturbations.}
    (a) \dubinsra{}: rollouts of a safety-filtered car from one start under time-correlated steering noise; rollouts that hit the obstacle are red.
    (b) \crowdnav{}: once inside the ego's sensing radius (dotted), an agent is briefly pushed toward the ego (orange).
    (c) \posswap{}: agents (discs) drive to their goals (rings) while their executed acceleration lags the command (inset).}
    \label{fig:task_overview}
\end{figure}

\noindent\textbf{Baselines.}
We compare against the following safety filters:
\begin{itemize}
    \item \textbf{HJ Reachability (\AlgHJ{})} \cite{choi2021cbvf, fisac2018general}: A learned HJ CBF.
    \item \textbf{HJ-Isaacs (\AlgHJI{})}\cite{hsunguyen2023isaacs}: The robust counterpart to \AlgHJ{} that learns a CBF under $d_{\max} \geq 0$ bounded disturbances.
    \item \textbf{Handcrafted CBF (\AlgHCBF{})} \cite{ames2016control}: A task-specific handcrafted CBF.
    \item \textbf{Robust Handcrafted CBF (\AlgRHCBF{})} \cite{jankovic2018robust}: Same as \AlgHCBF{}, but replaces \Cref{eq:cbf_qp:constraint} with its worst-case counterpart over box-bounded disturbances.
\end{itemize}

All methods are trained on a task-agnostic environment and evaluated on a perturbed task environment.
Adversary disturbance channels act through the ego controls, and additionally through other agent controls in multi-agent settings.
Each baseline evaluation runs for $200$ seeds per configuration in $\alpha \in \{0.1, 0.5\}$ and $20$ evenly-spaced $D \in [0, D_{\mathrm{ceiling}}]$ where $D_{\mathrm{ceiling}}$ is the maximum measured value of each baseline CBF.
More specifically, for a given $D$, the safety filter will use the baseline CBF $B(x)$ as $B(x) - D$ instead.
All runs for \AlgHJI{} and \AlgRHCBF{} use $d_{\max}=0.5$ by default.
For LEAP, we use the switching fallback policy with QP filtering as the nominal \Cref{thm:switch_qp} for simplicity as the disturbance-budget QP filter \eqref{eq:budget_qp} requires tracking cumulative disturbance effort.
We visualize the simulation tasks and their test-time disturbances in \Cref{fig:task_overview}. 

\subsection{Noisy Dubins Reach-Avoid}

\begin{figure}[H]
    \centering
    \includegraphics[width=1.0\linewidth]{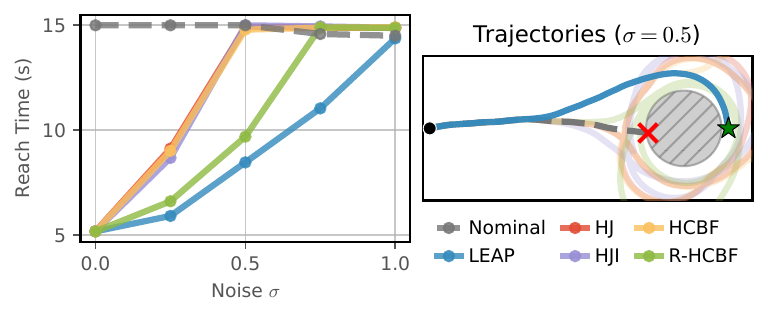}
    \caption{\textbf{Performance comparison on the noisy \dubinsra{} task.}
    Left: fastest goal-reaching time among filter configurations achieving at least $95\%$ safety, versus noise level $\sigma$.
    Right: one filtered trajectory per baseline.
    \AlgLEAP{} maintains a high safety rate and reaches the goal faster than the baselines as noise increases.
    }
    \label{fig:dubins_summary}
\end{figure}

In \dubinsra, a constant-speed Dubins car must steer to a goal while avoiding a circular obstacle.
Its state is $x=(p_x,p_y,\theta)$ with dynamics $\dot p_x=v\cos\theta$, $\dot p_y=v\sin\theta$, and $\dot\theta=\omega_{\max}\operatorname{clip}(u+d_{\max}d,-1,1)$, where $v$ is fixed and $u,d\in[-1,1]$ are the normalized steering command and disturbance.

\begin{equation}
    \begin{aligned}
        e_t &= \rho e_{t-1}+\sigma\sqrt{1-\rho^2}\,\xi_t,
        & \xi_t&\sim\mathcal{N}(0,1),\\
        u_t^{\mathrm{exec}}&=\operatorname{clip}(u_t+e_t,-1,1),
        & \dot\theta_t&=\omega_{\max}u_t^{\mathrm{exec}},\\
        \rho&=\exp(-\Delta t/\tau), & e_0&=0.
    \end{aligned}
    \label{eq:dubins_eval_noise}
\end{equation}

With correlation time $\tau=0.5\,\mathrm{s}$, \AlgLEAP{} reaches the goal faster than the baselines at the $95\%$ safety threshold as noise increases (\Cref{fig:dubins_summary}).
The example trajectories show \AlgLEAP{} passing close to the obstacle at a safe heading.
Because the disturbance effort depends on the ego's heading near the obstacle,
\AlgLEAP{} permits close passage when redirecting the car into collision would require substantial effort.

\subsection{Adversarial Crowd Navigation}

\begin{figure}[H]
    \centering
    \subfloat[Safety--completion tradeoff.\label{fig:crowd_pareto_goal_rate}]{\includegraphics[width=0.49\linewidth]{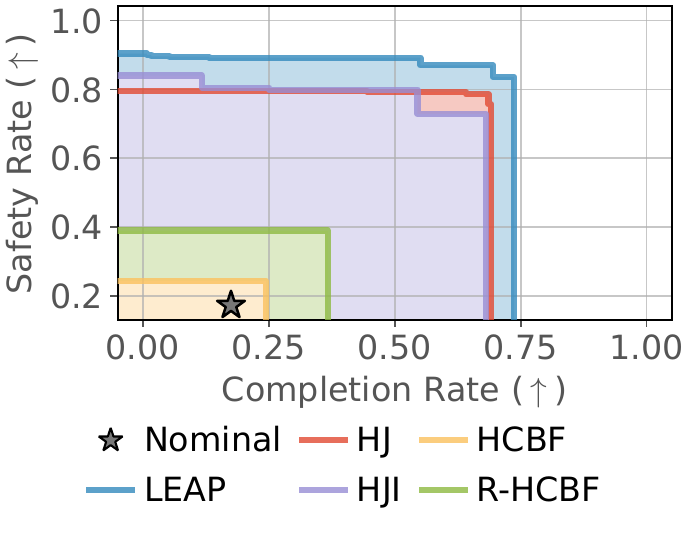}}
\hfill
    \subfloat[\AlgHJI{} disturbance sensitivity.\label{fig:crowd_pareto_goal_rate_hjiladder}]{\includegraphics[width=0.49\linewidth]{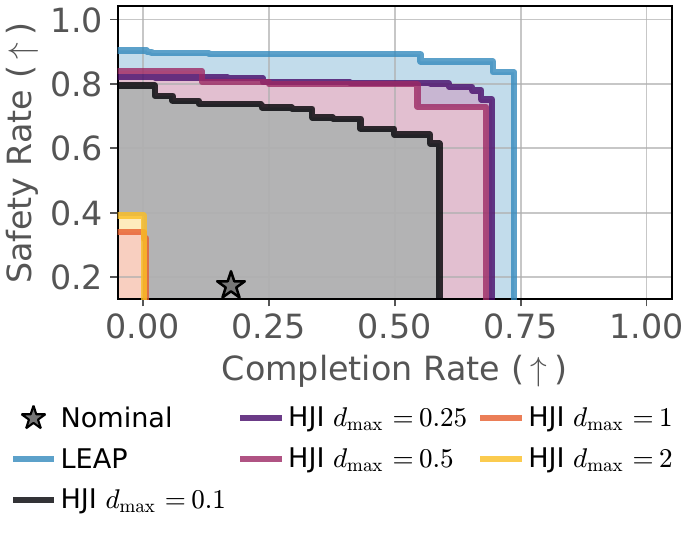}}
\caption{\textbf{Performance comparison on the \crowdnav{} task.}
    (a) Safety--completion Pareto frontiers over $\alpha$ and $D$.
    \AlgLEAP{}'s frontier dominates the baselines, achieving a better safety--completion tradeoff.
    (b) Sensitivity of \AlgHJI{} to its disturbance bound $d_{\max}$, compared with \AlgLEAP{} using $d_{\max}=\infty$.
    \AlgHJI{} requires tuning: small bounds provide insufficient robustness, whereas large bounds are overly conservative.}
    \label{fig:crowd_pareto_goal_rate_comparison}
\end{figure}

In the \crowdnav{} task, a unicycle ego must move through a crowd of five agents, whose nominal speed of $0.9\,\mathrm{m/s}$ exceeds the ego’s maximum speed of $0.6\,\mathrm{m/s}$, without colliding.
The ego and surrounding agents are represented by a graph neural network (GNN) whose node and edge features encode the relative states of the ego and nearby agents within a sensing radius.
At evaluation time, the ego's task is to reach a finish line at $y=12\,\mathrm{m}$ starting from $y=0\,\mathrm{m}$.
After an agent enters the ego's sensing radius, an adversary can briefly perturb its motion toward the ego.
We define safety rate as the fraction of episodes without a collision and completion rate as the fraction in which the ego crosses the finish line.

\AlgLEAP{} achieves a better safety--completion tradeoff than the baselines, including all evaluated \AlgHJI{} disturbance bounds (\Cref{fig:crowd_pareto_goal_rate_comparison}).
Small \AlgHJI{} bounds provide insufficient robustness, whereas large bounds restrict progress; \AlgLEAP{} uses a single $d_{\max}=\infty$.

\subsection{Tracking Latency in Position Swapping}

\begin{figure}[H]
    \centering
    \includegraphics[width=\linewidth]{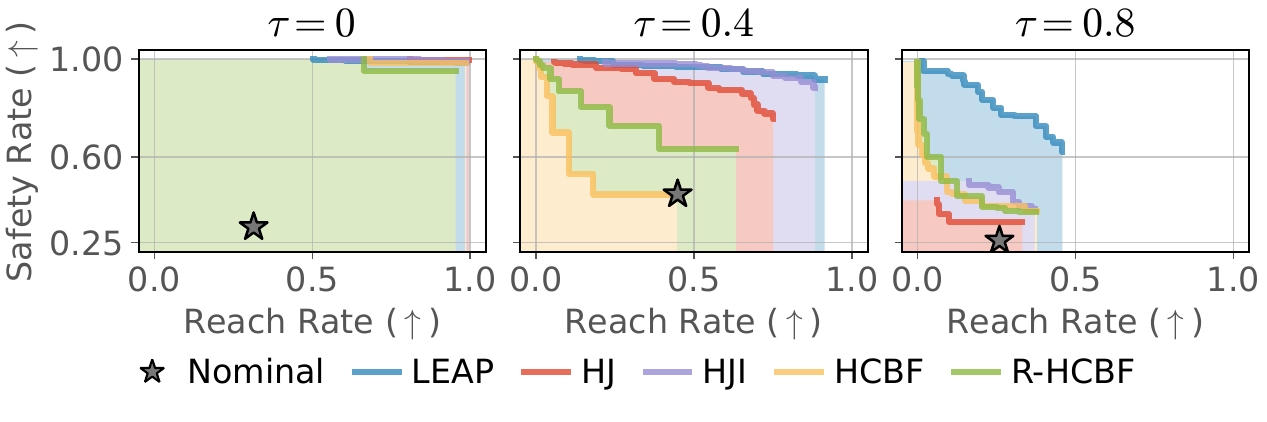}
    \caption{\textbf{Performance comparison on the \posswap{} task.}
    Safety--reach Pareto frontiers under increasing tracking latency $\tau$, which is absent during training.
    \AlgLEAP{} remains the most robust to increasing $\tau$ despite not observing it during training.
    }
    \label{fig:crowdv2_hwv3_pareto_reach}
\end{figure}

In the \posswap{} task, six double-integrator agents start at rest and are assigned one another's starting positions as goals; they must then swap positions without colliding.
CBF-based multirobot collision avoidance~\cite{wang2017safety} and learned graph certificates such as GCBF+~\cite{zhang2025gcbfplus} provide context for this task; here we evaluate robustness to tracking latency at a fixed team size.
Agents are again represented by a GNN with node features identifying object type and velocity, and edge features summarizing relative position, velocity, and distance.
At evaluation time, each agent $i$ employs a goal-tracking PD controller and executes in a first-order actuator model with time latency variable $\tau \geq 0$ mapping the commanded to executed acceleration via

\begin{equation}
    \begin{aligned}
        \lambda_\tau
        &=\begin{cases}
              0, & \tau=0,\\
              e^{-\Delta t/\tau}, & \tau>0,
        \end{cases}\\
        a^{\mathrm{exec}}_{i,t+1}
        &=\lambda_\tau a^{\mathrm{exec}}_{i,t}
        +(1-\lambda_\tau)a^{\mathrm{cmd}}_{i,t},\\
    \end{aligned}
    \label{eq:posswap_latency_dynamics}
\end{equation}

where $\tau=0$ introduces no tracking latency.

\AlgLEAP{} maintains the best safety--reach tradeoff as latency increases; \AlgHJI{} is competitive at moderate latency but degrades sharply at larger $\tau$ (\Cref{fig:crowdv2_hwv3_pareto_reach}).

\section{Hardware Experiments}
\label{sec:hardware}
We evaluate robustness to real-world disturbances in the \quadrupeddrone{} task.
A Unitree Go2 quadruped, modeled as a unicycle, follows a nominal controller through four waypoints while avoiding a Crazyflie drone~\cite{giernacki2017crazyflie}.
We localize both robots using the Crazyflie lighthouse system and send their commands from a centralized computer.
The safety filter controls only the Go2; the drone follows a separate controller that randomly changes its heading toward the Go2.
We treat these unexpected heading changes as disturbances during evaluation.

During training, the drone is modeled as a unicycle with constant heading and speed.
Neither the drone's evaluation behavior nor the nominal waypoint controller is available to the baseline methods during training; the waypoint controller is introduced only at deployment through the safety-filtering framework.
Thus, the hardware task tests robustness to departures from the modeled drone motion while preserving waypoint progress.

\Cref{fig:hw_bars} reports safety rate and distance to the final waypoint after a $30\,\mathrm{s}$ time limit.
\Cref{fig:hw_example} illustrates an unexpected drone turn into the Go2's path: \AlgHJ{} collides, whereas \AlgLEAP{} avoids the collision.

\begin{figure}[t]
 \centering
 \includegraphics[width=\linewidth]{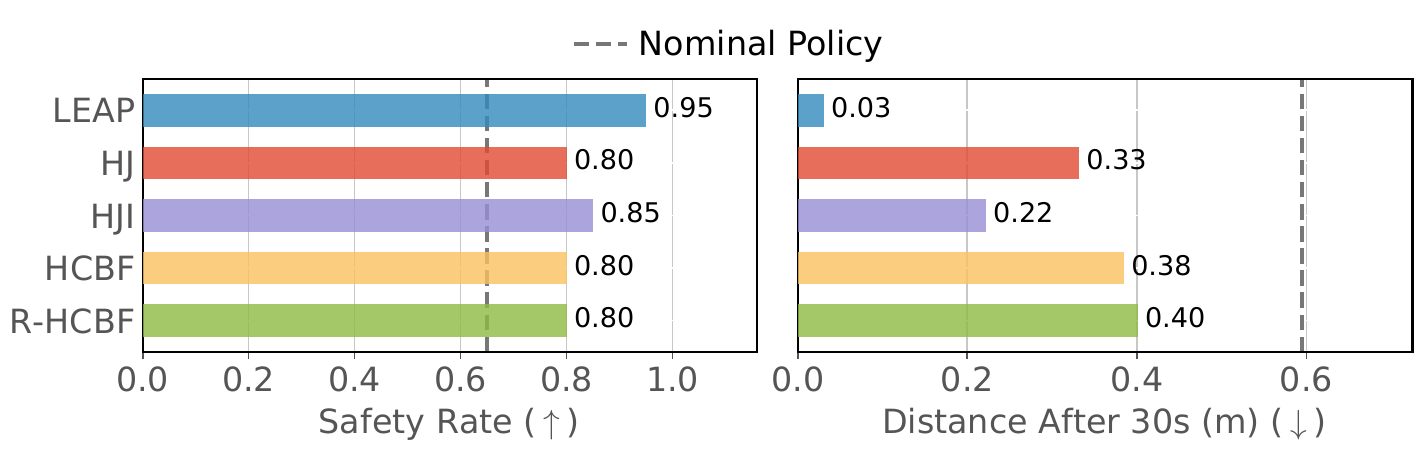}
 \caption{\textbf{Performance comparison on hardware \quadrupeddrone{} task.}
    Safety rate and final-waypoint distance after $30\,\mathrm{s}$.
    \AlgLEAP{} performs best, followed by \AlgHJI{}.
 }
 \label{fig:hw_bars}
\end{figure}
\begin{figure}[t]
 \centering
 \includegraphics[width=\linewidth]{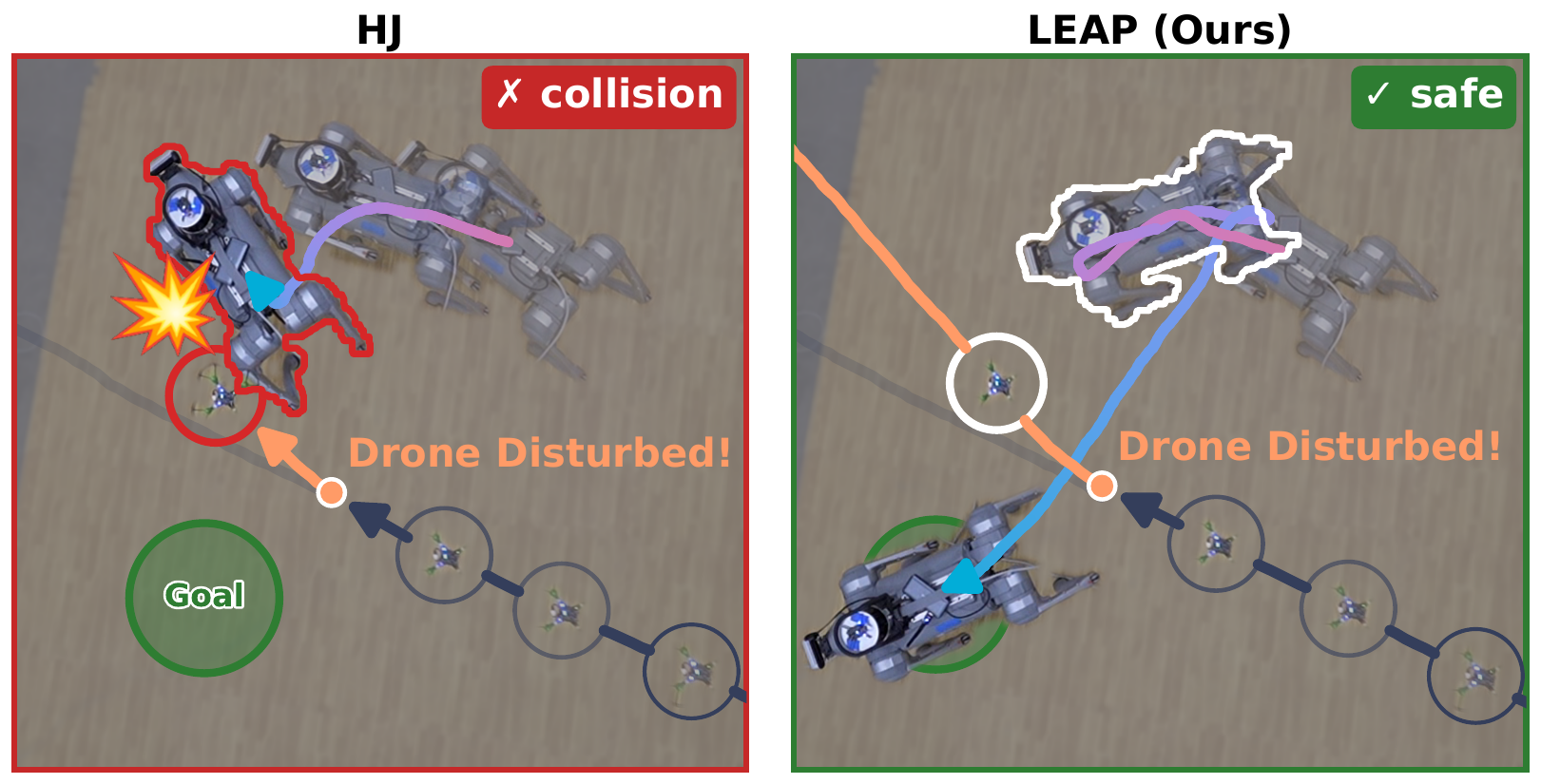}
 \caption{\textbf{Comparing \AlgHJ{} and \AlgLEAP{} on hardware.}
    When the Crazyflie turns into the Go2's path, \AlgHJ{} collides while \AlgLEAP{} avoids collision.
 }
\label{fig:hw_example}
\end{figure}

 \section{Conclusion}
\label{sec:conclusion}
We introduced Least-Effort Adversarial Potential (LEAP), which quantifies robustness as the minimum cumulative disturbance effort required to cause failure under a fixed fallback policy.
This gives the certificate's magnitude an operational meaning tied to the dynamics and disturbance model.
Under suitable regularity assumptions, positive-threshold shifts of the exact LEAP value are CBFs for the nominal system, and LEAP-based filters guarantee safety under a cumulative disturbance budget.
We developed an alternating reinforcement learning procedure to approximate LEAP and improve its fallback policy.
Simulation and hardware experiments demonstrate improved safety--task performance tradeoffs under control noise, adversarial interactions, and tracking latency.

\noindent\textbf{Limitations. }
Computing LEAP requires solving a challenging minimum-cost reachability problem.
Our classifier-based approximation is heuristic and is not guaranteed to converge to the exact value; the theoretical safety guarantees therefore do not automatically extend to the learned certificate.
More principled solvers, such as epigraph methods~\cite{so2024solving}, and verification of learned certificates are promising directions.
LEAP also depends on the quality of its fallback policy.
Although our alternating fallback--adversary training was stable empirically, its convergence remains unproven.
Future work will investigate these approximation and training limitations, as well as connections to other robustness properties.

 \bibliographystyle{IEEEtran}
\bibliography{ref}

@IEEEtranBSTCTL{noUrl,
    CTLuse_url = "no"
}

@inproceedings{so2024train,
    title = {How to train your neural control barrier function: Learning safety filters for complex input-constrained systems},
    author = {So, Oswin and Serlin, Zachary and Mann, Makai and Gonzales, Jake and Rutledge, Kwesi and Roy, Nicholas and Fan, Chuchu},
    booktitle = {2024 IEEE International Conference on Robotics and Automation (ICRA)},
    pages = {11532--11539},
    year = {2024},
    organization = {IEEE}
}

@misc{tan2023value,
    title = {Value Functions are Control Barrier Functions: Verification of Safe Policies using Control Theory},
    author = {Tan, Daniel C. H. and Acero, Fernando and McCarthy, Robert and Kanoulas, Dimitrios and Li, Zhibin},
    year = {2023},
    eprint = {2306.04026},
    archivePrefix = {arXiv},
    primaryClass = {cs.LG},
    url = {https://arxiv.org/abs/2306.04026}
}

@inproceedings{teoh2026madr,
    title = {{MADR}: {MPC}-guided Adversarial {DeepReach}},
    author = {Teoh, Ryan and Tonkens, Sander and Sharpless, William and Yang, Aijia and Feng, Zeyuan and Bansal, Somil and Herbert, Sylvia},
    booktitle = {2026 IEEE International Conference on Robotics and Automation (ICRA)},
    year = {2026},
    url = {https://arxiv.org/abs/2510.18845}
}

@inproceedings{dean2021guaranteeing,
    title = {Guaranteeing Safety of Learned Perception Modules via Measurement-Robust Control Barrier Functions},
    author = {Dean, Sarah and Taylor, Andrew J. and Cosner, Ryan K. and Recht, Benjamin and Ames, Aaron D.},
    booktitle = {Proceedings of the 2020 Conference on Robot Learning},
    series = {Proceedings of Machine Learning Research},
    volume = {155},
    pages = {654--670},
    year = {2021},
    publisher = {PMLR},
    url = {https://proceedings.mlr.press/v155/dean21a.html}
}

@article{yin2020reachability,
    title = {Reachability Analysis Using Dissipation Inequalities for Uncertain Nonlinear Systems},
    author = {Yin, He and Packard, Andrew and Arcak, Murat and Seiler, Peter},
    journal = {Systems \& Control Letters},
    volume = {142},
    pages = {104736},
    year = {2020},
    doi = {10.1016/j.sysconle.2020.104736},
    url = {https://arxiv.org/abs/1808.02585}
}

@article{yin2021backward,
    title = {Backward Reachability Using Integral Quadratic Constraints for Uncertain Nonlinear Systems},
    author = {Yin, He and Seiler, Peter and Arcak, Murat},
    journal = {IEEE Control Systems Letters},
    volume = {5},
    number = {2},
    pages = {707--712},
    year = {2021},
    doi = {10.1109/LCSYS.2020.3005315},
    url = {https://arxiv.org/abs/2003.05617}
}

@article{alan2022safe,
    title = {Safe Controller Synthesis With Tunable Input-to-State Safe Control Barrier Functions},
    author = {Alan, Anil and Taylor, Andrew J. and He, Chaozhe R. and Orosz, G{\'a}bor and Ames, Aaron D.},
    journal = {IEEE Control Systems Letters},
    volume = {6},
    pages = {908--913},
    year = {2022},
    doi = {10.1109/LCSYS.2021.3087443},
    url = {https://arxiv.org/abs/2103.08041}
}

@article{wabersich2021predictive,
    title = {A Predictive Safety Filter for Learning-Based Control of Constrained Nonlinear Dynamical Systems},
    author = {Wabersich, Kim Peter and Zeilinger, Melanie N.},
    journal = {Automatica},
    volume = {129},
    pages = {109597},
    year = {2021},
    doi = {10.1016/j.automatica.2021.109597}
}

@article{wabersich2023predictive,
    title = {Predictive Control Barrier Functions: Enhanced Safety Mechanisms for Learning-Based Control},
    author = {Wabersich, Kim Peter and Zeilinger, Melanie N.},
    journal = {IEEE Transactions on Automatic Control},
    volume = {68},
    number = {5},
    pages = {2638--2651},
    year = {2023},
    doi = {10.1109/TAC.2022.3175628},
    url = {https://arxiv.org/abs/2105.10241}
}

@inproceedings{pinto2017robust,
    title = {Robust Adversarial Reinforcement Learning},
    author = {Pinto, Lerrel and Davidson, James and Sukthankar, Rahul and Gupta, Abhinav},
    booktitle = {Proceedings of the 34th International Conference on Machine Learning},
    series = {Proceedings of Machine Learning Research},
    volume = {70},
    pages = {2817--2826},
    year = {2017},
    publisher = {PMLR},
    url = {https://proceedings.mlr.press/v70/pinto17a.html}
}

@inproceedings{taylor2020learning,
    title = {Learning for Safety-Critical Control with Control Barrier Functions},
    author = {Taylor, Andrew and Singletary, Andrew and Yue, Yisong and Ames, Aaron},
    booktitle = {Proceedings of the 2nd Conference on Learning for Dynamics and Control},
    series = {Proceedings of Machine Learning Research},
    volume = {120},
    pages = {708--717},
    year = {2020},
    publisher = {PMLR},
    url = {https://proceedings.mlr.press/v120/taylor20a.html}
}

@inproceedings{bansal2021deepreach,
    title = {{DeepReach}: A Deep Learning Approach to High-Dimensional Reachability},
    author = {Bansal, Somil and Tomlin, Claire J.},
    booktitle = {2021 IEEE International Conference on Robotics and Automation (ICRA)},
    year = {2021},
    url = {https://arxiv.org/abs/2011.02082}
}

@article{wang2017safety,
    title = {Safety Barrier Certificates for Collisions-Free Multirobot Systems},
    author = {Wang, Li and Ames, Aaron D. and Egerstedt, Magnus},
    journal = {IEEE Transactions on Robotics},
    volume = {33},
    number = {3},
    pages = {661--674},
    year = {2017},
    doi = {10.1109/TRO.2017.2659727}
}

@article{zhang2025gcbfplus,
    title = {{GCBF+}: A Neural Graph Control Barrier Function Framework for Distributed Safe Multiagent Control},
    author = {Zhang, Songyuan and So, Oswin and Garg, Kunal and Fan, Chuchu},
    journal = {IEEE Transactions on Robotics},
    volume = {41},
    pages = {1533--1552},
    year = {2025},
    doi = {10.1109/TRO.2025.3530348},
    url = {https://github.com/MIT-REALM/gcbfplus}
}

@article{ames2016control,
    title = {Control barrier function based quadratic programs for safety critical systems},
    author = {Ames, Aaron D and Xu, Xiangru and Grizzle, Jessy W and Tabuada, Paulo},
    journal = {IEEE transactions on automatic control},
    volume = {62},
    number = {8},
    pages = {3861--3876},
    year = {2016},
    publisher = {IEEE}
}

@inproceedings{ames2019cbfsurvey,
    author = {Ames, Aaron D. and Coogan, Samuel and Egerstedt, Magnus and Notomista, Gennaro and Sreenath, Koushil and Tabuada, Paulo},
    booktitle = {2019 18th European Control Conference (ECC)},
    title = {Control Barrier Functions: Theory and Applications},
    year = {2019},
    volume = {},
    number = {},
    pages = {3420-3431},
    doi = {10.23919/ECC.2019.8796030} }

@inproceedings{choi2021cbvf,
    author = {Choi, Jason J. and Lee, Donggun and Sreenath, Koushil and Tomlin, Claire J. and Herbert, Sylvia L.},
    title = {Robust Control Barrier-Value Functions for Safety-Critical Control},
    year = {2021},
    publisher = {IEEE Press},
    url = {https://doi.org/10.1109/CDC45484.2021.9683085},
    doi = {10.1109/CDC45484.2021.9683085},
    booktitle = {2021 60th IEEE Conference on Decision and Control (CDC)},
    pages = {6814-6821},
    numpages = {8},
    location = {Austin, TX, USA}
}

@article{oh2026synthesis,
    title = {Synthesis and Deployment of Maximal Robust Control Barrier Functions through Adversarial Reinforcement Learning},
    author = {Oh, Donggeon David and Nguyen, Duy P and Hu, Haimin and Fisac, Jaime Fern{\'a}ndez},
    journal = {arXiv preprint arXiv:2604.13192},
    year = {2026}
}

@article{xu2015robustness,
    title = {Robustness of Control Barrier Functions for Safety Critical Control},
    author = {Xu, Xiangru and Tabuada, Paulo and Grizzle, Jessy W and Ames, Aaron D},
    journal = {IFAC-PapersOnLine},
    volume = {48},
    number = {27},
    pages = {54--61},
    year = {2015},
    publisher = {Elsevier},
    url = {http://ames.caltech.edu/ADHS15_Final.pdf}
}

@inproceedings{nguyen2024gameplayfiltersrobustzeroshot,
    title = {Gameplay Filters: Robust Zero-Shot Safety through Adversarial Imagination},
    author = {Duy P. Nguyen and Kai-Chieh Hsu and Wenhao Yu and Jie Tan and Jaime F. Fisac},
    booktitle = {Proceedings of The 8th Conference on Robot Learning},
    series = {Proceedings of Machine Learning Research},
    volume = {270},
    pages = {387--407},
    year = {2025},
    publisher = {PMLR},
    url = {https://proceedings.mlr.press/v270/nguyen25a.html},
}

@inproceedings{hsunguyen2023isaacs,
    title = {ISAACS: Iterative Soft Adversarial Actor-Critic for Safety},
    author = {Hsu, Kai-Chieh and Nguyen, Duy Phuong and Fisac, Jaime Fern\`andez},
    booktitle = {Proceedings of the 5th Annual Learning for Dynamics and Control Conference},
    pages = {90--103},
    year = {2023},
    editor = {Matni, Nikolai and Morari, Manfred and Pappas, George J.},
    volume = {211},
    series = {Proceedings of Machine Learning Research},
    month = {15--16 Jun},
    publisher = {PMLR},
    url = {https://proceedings.mlr.press/v211/hsu23a.html}
}

@article{seiler2021control,
    title = {Control barrier functions with unmodeled input dynamics using integral quadratic constraints},
    author = {Seiler, Peter and Jankovic, Mrdjan and Hellstrom, Erik},
    journal = {IEEE Control Systems Letters},
    volume = {6},
    pages = {1664--1669},
    year = {2022},
    doi = {10.1109/LCSYS.2021.3130782},
    url = {https://arxiv.org/abs/2108.10491},
    publisher = {IEEE}
}

@article{knoedler2025safety,
    title = {Safety on the fly: Constructing robust safety filters via policy control barrier functions at runtime},
    author = {Knoedler, Luzia and So, Oswin and Yin, Ji and Black, Mitchell and Serlin, Zachary and Tsiotras, Panagiotis and Alonso-Mora, Javier and Fan, Chuchu},
    journal = {IEEE Robotics and Automation Letters},
    year = {2025},
    publisher = {IEEE}
}

@article{fisac2018general,
    title = {A general safety framework for learning-based control in uncertain robotic systems},
    author = {Fisac, Jaime F and Akametalu, Anayo K and Zeilinger, Melanie N and Kaynama, Shahab and Gillula, Jeremy and Tomlin, Claire J},
    journal = {IEEE Transactions on Automatic Control},
    volume = {64},
    number = {7},
    pages = {2737--2752},
    year = {2018},
    publisher = {IEEE}
}

@article{kolathaya2019input,
    title = {Input-to-State Safety With Control Barrier Functions},
    author = {Kolathaya, Shishir and Ames, Aaron D.},
    journal = {IEEE Control Systems Letters},
    volume = {3},
    number = {1},
    pages = {108--113},
    year = {2019},
    doi = {10.1109/LCSYS.2018.2853698},
    url = {https://arxiv.org/abs/1803.03035}
}

@inproceedings{chen2021backup,
    title = {Backup Control Barrier Functions: Formulation and Comparative Study},
    author = {Chen, Yuxiao and Jankovic, Mrdjan and Santillo, Mario and Ames, Aaron D.},
    booktitle = {2021 60th IEEE Conference on Decision and Control (CDC)},
    year = {2021},
    url = {https://arxiv.org/abs/2104.11332}
}

@article{lee2020adaptive,
    title = {Adaptive Stress Testing: Finding Likely Failure Events with Reinforcement Learning},
    author = {Lee, Ritchie and Mengshoel, Ole J. and Saksena, Anshu and Gardner, Ryan and Genin, Daniel and Silbermann, Joshua and Owen, Michael and Kochenderfer, Mykel J.},
    journal = {Journal of Artificial Intelligence Research},
    year = {2020},
    url = {https://arxiv.org/abs/1811.02188}
}

@article{lygeros2004reachability,
    title = {On Reachability and Minimum Cost Optimal Control},
    author = {Lygeros, John},
    journal = {Automatica},
    volume = {40},
    number = {6},
    pages = {917--927},
    year = {2004},
    doi = {10.1016/j.automatica.2004.01.012}
}

@misc{liao2021costlimited,
    title = {Computation of Reachable Sets Based on {Hamilton-Jacobi-Bellman} Equation with Running Cost Function},
    author = {Liao, Weiwei and Liang, Tao},
    year = {2021},
    eprint = {2107.11941},
    archivePrefix = {arXiv},
    url = {https://arxiv.org/abs/2107.11941}
}

@inproceedings{tonkens2022refining,
    title = {Refining Control Barrier Functions through {Hamilton-Jacobi} Reachability},
    author = {Tonkens, Sander and Herbert, Sylvia},
    booktitle = {2022 IEEE/RSJ International Conference on Intelligent Robots and Systems (IROS)},
    year = {2022},
    url = {https://arxiv.org/abs/2204.12507}
}

@inproceedings{tonkens2024patching,
    title = {Patching Approximately Safe Value Functions Leveraging Local {Hamilton-Jacobi} Reachability Analysis},
    author = {Tonkens, Sander and Toofanian, Alex and Qin, Zhizhen and Gao, Sicun and Herbert, Sylvia},
    booktitle = {2024 IEEE 63rd Conference on Decision and Control (CDC)},
    year = {2024},
    url = {https://arxiv.org/abs/2304.09850}
}

@article{hsu2023safety,
    title = {The safety filter: A unified view of safety-critical control in autonomous systems},
    author = {Hsu, Kai-Chieh and Hu, Haimin and Fisac, Jaime F},
    journal = {Annual Review of Control, Robotics, and Autonomous Systems},
    volume = {7},
    year = {2023},
    publisher = {Annual Reviews}
}

@inproceedings{lyu2025integral,
    title = {Integral Input-to-State Safe Barrier Functions},
    author = {Lyu, Ziliang and Fang, Xu and Yuan, Heling and Li, Xiuxian and Hong, Yiguang and Xie, Lihua},
    booktitle = {2025 IEEE 64th Conference on Decision and Control (CDC)},
    pages = {3901--3906},
    year = {2025},
    organization = {IEEE},
    doi = {10.1109/CDC57313.2025.11312407}
}

@article{so2024solving,
    title = {Solving minimum-cost reach avoid using reinforcement learning},
    author = {So, Oswin and Ge, Cheng and Fan, Chuchu},
    journal = {Advances in Neural Information Processing Systems},
    volume = {37},
    pages = {30951--30984},
    year = {2024}
}

@inproceedings{bansal2017hamilton,
    title = {Hamilton-jacobi reachability: A brief overview and recent advances},
    author = {Bansal, Somil and Chen, Mo and Herbert, Sylvia and Tomlin, Claire J},
    booktitle = {2017 IEEE 56th annual conference on decision and control (CDC)},
    pages = {2242--2253},
    year = {2017},
    organization = {IEEE}
}

@book{bardi1997optimal,
    title = {Optimal control and viscosity solutions of Hamilton-Jacobi-Bellman equations},
    author = {Bardi, Martino and Dolcetta, Italo Capuzzo and others},
    volume = {12},
    year = {1997},
    publisher = {Springer}
}

@article{schulman2017proximal,
    title = {Proximal policy optimization algorithms},
    author = {Schulman, John and Wolski, Filip and Dhariwal, Prafulla and Radford, Alec and Klimov, Oleg},
    journal = {arXiv preprint arXiv:1707.06347},
    year = {2017}
}

@inproceedings{giernacki2017crazyflie,
    title = {Crazyflie 2.0 quadrotor as a platform for research and education in robotics and control engineering},
    author = {Giernacki, Wojciech and Skwierczy{\'n}ski, Mateusz and Witwicki, Wojciech and Wro{\'n}ski, Pawe{\l} and Kozierski, Piotr},
    booktitle = {2017 22nd international conference on methods and models in automation and robotics (MMAR)},
    pages = {37--42},
    year = {2017},
    organization = {IEEE}
}

@article{thrun2002probabilistic,
    title = {Probabilistic robotics},
    author = {Thrun, Sebastian},
    journal = {Communications of the ACM},
    volume = {45},
    number = {3},
    pages = {52--57},
    year = {2002},
    publisher = {ACM New York, NY, USA}
}

@article{jankovic2018robust,
    title = {Robust control barrier functions for constrained stabilization of nonlinear systems},
    journal = {Automatica},
    volume = {96},
    pages = {359-367},
    year = {2018},
    issn = {0005-1098},
    doi = {10.1016/j.automatica.2018.07.004},
    url = {https://www.sciencedirect.com/science/article/pii/S0005109818303509},
    author = {Mrdjan Jankovic},
}

@inbook{sutton2018nstep,
  author    = {Sutton, Richard S. and Barto, Andrew G.},
  title     = {Reinforcement Learning: An Introduction},
  chapter   = {7},
  pages     = {168-172},
  publisher = {MIT Press},
  year      = {2018},
  edition   = {Second},
  address   = {Cambridge, MA}
}

@article{sorensen2026time,
  title={A Time-to-Collision Barrier Function Approach to Collision Avoidance for Stochastic Systems},
  author={Sorensen, Benedikt Barthel and Black, Mitchell and Noorani, Erfaun and Sapsis, Themistoklis P},
  journal={arXiv preprint arXiv:2609.17347},
  year={2026}
}
\clearpage
\appendices
\end{document}